\documentclass[11pt]{article}

\usepackage[left=1in,right=1in,bottom=1in,top=1in]{geometry}
\usepackage{amsmath,amssymb,amsthm}
\usepackage{natbib}
\usepackage{graphicx}
\graphicspath{{}}
\usepackage{xcolor}
\usepackage{hyperref}
\hypersetup{colorlinks=true, citecolor=black, linkcolor=., urlcolor=cyan}
\usepackage{adjustbox}
\usepackage{booktabs}
\usepackage[linesnumbered,ruled]{algorithm2e}
\usepackage{enumerate}
\usepackage{multirow}

\theoremstyle{definition}
\newtheorem{definition}{Definition}[section]
\newtheorem{theorem}{Theorem}[section]

\renewcommand{\abstract}[1]{
 \centerline{
 \begin{minipage}{0.7\linewidth}
 \hrule
 \vskip 0.1in
  \begin{center}
    {\bf Abstract}
  \end{center}
  #1
 \vskip 0.1in
 \hrule
 \end{minipage}}
 \vskip 0.3in}

\providecommand{\I}{\mathbf{I}}

\renewcommand{\r}{\mathbf{r}}

\renewcommand{\v}{\mathbf{v}}

\providecommand{\x}{\mathbf{x}}
\providecommand{\X}{\mathbf{X}}
\providecommand{\y}{\mathbf{y}}

\providecommand{\zero}{\mathbf{0}}

\providecommand{\lam}{\lambda}

\providecommand{\bb}{\boldsymbol{\beta}}

\providecommand{\bep}{\boldsymbol{\epsilon}}

\providecommand{\bt}{\boldsymbol{\theta}}

\providecommand{\abs}[1]{\left\lvert#1\right\rvert}
\providecommand{\norm}[1]{\left\lVert#1\right\rVert}

\providecommand{\|}{\Vert}

\newcommand{\argmin}{\operatornamewithlimits{argmin}}
\newcommand{\argmax}{\operatornamewithlimits{argmax}}

\title{Hybrid Safe-Strong Rules for Efficient Optimization in Lasso-Type Problems}
\author{Yaohui Zeng\\Department of Biostatistics\\University of Iowa
  \and
  Tianbao Yang\\Department of Computer Science\\University of Iowa
  \and
  Patrick Breheny\\Department of Biostatistics\\University of Iowa}
\date{\today}

\begin{document}

\maketitle

\abstract{The lasso model has been widely used for model selection in data mining, machine learning, and high-dimensional statistical analysis.
However, with the ultrahigh-dimensional, large-scale data sets now collected in many real-world applications, it is important to develop algorithms to solve the lasso that efficiently scale up to problems of this size.
Discarding features from certain steps of the algorithm is a powerful technique for increasing efficiency and addressing the Big Data challenge. In this paper, we propose a family of hybrid safe-strong rules (HSSR) which incorporate safe screening rules into the sequential strong rule (SSR) to remove unnecessary computational burden. In particular, we present two instances of HSSR, namely SSR-Dome and SSR-BEDPP, for the standard lasso problem. We further extend SSR-BEDPP to the elastic net and group lasso problems to demonstrate the generalizability of the hybrid screening idea.
Extensive numerical experiments with synthetic and real data sets are conducted for both the standard lasso and the group lasso problems. Results show that our proposed hybrid rules can substantially outperform existing state-of-the-art rules.}

\section{Introduction} \label{sect_intro}

The lasso model \citep{Tibshirani1996} is widely used in data mining, machine learning, and high-dimensional statistics. The model is defined as the following optimization problem
\begin{align} \label{linear_lasso}
\widehat{\bb}(\lambda) = \argmin_{\bb \in \mathbb{R}^p} \frac{1}{2n} \norm{\y - \X \bb}_2^2 + \lambda \norm{\bb}_1,
\end{align}
where $\y$ is the $n \times 1$ response vector, $\X = (\x_1, \ldots, \x_p)$ is the $n \times p$ feature matrix, $\bb \in \mathbb{R}^p$ is the coefficient vector, and $\lambda \geq 0$ is a regularization parameter. $\| \cdot \|$ and $\|\cdot\|_1$ respectively denote the Euclidean ($\ell_2$) norm and $\ell_1$ norm.

Due to its property of automatic feature selection, the lasso model has attracted extensive studies with a wide range of successful applications to many areas, such as signal processing \citep{angelosante2009rls}, gene expression data analysis \citep{huang2003linear}, face recognition \citep{wright2009robust}, text mining \citep{li2015relevance} and so on. 
Efficiently solving the lasso model is therefore of great significance to statistical and machine learning practice.


Over the past years a number of efficient algorithms have been developed for solving the lasso~\citep{efron2004least, kim2007interior, garrigues2009homotopy, boyd2011distributed, wu2008coordinate, friedman2007pathwise, shalev2011stochastic}. Among them the pathwise coordinate descent algorithm~\citep{friedman2007pathwise} is simple, fast, and able to make use of the sparsity structure of the lasso and ``warm start'' strategy, making it very suitable and efficient to scale up to high-dimensional lasso problems~\citep{friedman2010regularization}. With the evolving era of Big Data, however, it is increasingly common to encounter large-scale, ultrahigh-dimensional data sets.  The increased number of features and observations in these data sets present added challenges to solving the lasso efficiently.

One idea for reducing computation time is drop certain features from the analysis prior to fitting the lasso.
As a result, the dimensionality of the feature matrix -- and hence the computational burden of the optimization -- will be substantially reduced.
This idea, known as \textit{feature screening}, has been around for a long time, but was first studied formally by \citet{fan2008sure}.
who studied the asymptotic properties of screening out features that have weak correlations with the response variable.
However, feature screening, which is usually based on the marginal relationship between a feature and the outcome, can incorrectly screen out important features and does not, therefore, solve the original optimization problem \eqref{linear_lasso}.

To avoid this problem, other researchers sought to develop {\em safe} rules that are guaranteed not to discard any active features.  These rules are usually based on exploiting geometric properties of the dual formulation of the lasso problem.  Their main idea is to bound the dual optimal solution $\widehat{\bt}(\lambda)$ of the lasso (formally defined in Section~\ref{sect_safe}) within a compact region $\boldsymbol \Theta$. Then given a feature $\x_j$, its coefficient estimate $\widehat{\beta}_j$ is guaranteed to be 0 if $\sup_{\bt \in \boldsymbol \Theta} |\x_j^T \bt| < \lambda$. This assertion is implied by the KKT condition: $|\x_j^T \bt(\lambda) | < \lambda \Rightarrow \beta_j = 0$ \citep{boyd2004convex}. The pioneering work in this direction is the SAFE rule developed by El Ghaoui et al. \citep{ghaoui2010safe}. The smaller the region $\boldsymbol \Theta$, the more features will be discarded and more efficiency gained; this has motivated other more powerful rules such as the EDPP rules \citep{JMLR:v16:wang15a}, the Dome test \citep{xiang2012fast}, and the Sphere tests \citep{xiang2011learning, Xiang2016lassoScreen}, which shrink $\boldsymbol \Theta$ according to different strategies.

A separate line of research has sought to develop ``strong'' rules that are more powerful at discarding features than safe rules and for which violations are unlikely, but possible.  This idea was initially proposed by \citet{tibshirani2012strong}, who developed \textit{sequential strong rules} (SSR) based upon the Karush-Kuhn-Tucker (KKT) conditions for the lasso problem along with an assumption of ``unit-slope'' bound.
The main idea is that we are still solving the original optimization problem, but we can skip certain calculations that are likely to be unnecessary, thereby reducing computational burden.
However, because it is possible for these rules to incorrectly discard active features, a post-convergence KKT checking step is required in order to guarantee the correctness of the solution.

In this paper, we propose combining safe and strong rules, yielding \textit{hybrid safe-strong rules} (HSSR) for discarding features in lasso-type problems. The key of HSSR is to incorporate \textit{simple yet safe} rules into SSR so as to remove a large amount of unnecessary post-convergence KKT checking on features that can be eliminated by safe rules. As a result, this paper will demonstrate that the total computing time for solving the lasso using these hybrid rules is substantially reduced compared to using either safe or strong rules alone.  Furthermore, the idea of HSSR provides a rather general feature screening framework since (i) in principle any safe rule can be combined with SSR, resulting in a more powerful rule; and (ii) HSSR can be easily extended to other lasso-type problems, either with different loss functions or  different regularization terms. In this paper we focus on three types of lasso problems with quadratic loss, namely, the standard lasso, the group lasso, and the elastic net.

Although this idea is relatively simple, we consider it to be novel for two primary reasons.  First, the existing literature is firmly divided and for the most part published in entirely different types of journals: most of the research on safe rules has appeared in machine learning and computer science journals, while the research on strong rules has appeared in statistics journals.  Most of what has been written gives the impression that these are two irreconcilable and mutually exclusive approaches to improving efficiency.  We show here that this is not the case -- the two types of rules can be combined in a relatively straightforward manner.  Second, the degree of efficiency gained by combining these rules is rather surprising, at least to us.  In many cases, the hybrid rules are more than the sum of their parts, providing much greater gains in efficiency when combined than using either type of rule alone.

The main contributions of this research include:
\begin{enumerate}
\item We propose a novel optimization framework for lasso screening that combines SSR with simple safe rules, resulting a family of hybrid safe-strong rules (HSSR) that are more efficient and scalable to large-scale data sets.
\item We develop two instances of HSSR, namely SSR-Dome and SSR-BEDPP, for feature screening in solving the lasso.
\item We extend SSR-BEDPP to two other lasso-type problems, the elastic-net \citep{zou2005regularization} and group lasso \citep{Yuan2006} to demonstrate the generalizability of the hybrid screening idea.
\item We evaluate the performance of newly proposed screening rules by extensive numerical experiments on both synthetic and real data sets, and show that our rules substantially outperform state-of-the-art ones.
\item We implement all screening rules in this paper in two publicly accessible R packages. Specifically, the rules for the standard lasso and elastic net are implemented in R package \texttt{biglasso}\footnote{\url{https://CRAN.R-project.org/package=biglasso}} \citep{zeng2017}, which aims to extend lasso model fitting to big data in R. The package \texttt{grpreg}\footnote{\url{https://CRAN.R-project.org/package=grpreg}}  \citep{Breheny2015} implements screening rules for the group lasso. The underlying optimization algorithm and screening rules in the R packages are implemented in C/C++ for fast computation.
\end{enumerate}

In this paper we assume without loss of generality that the response vector $\y$ is centered so that the intercept term is dropped from the lasso model. We further assume the feature vectors $\{\x_j\}_{j=1}^p$ are centered and standardized to have unit variance:
\begin{align} \label{stand}
\sum_{i=1}^n y_i = 0, \quad \sum_{i=1}^n x_{ij} = 0, \quad \frac{1}{n}\sum_{i=1}^n x_{ij}^2 = 1
\end{align}
for $j = 1, \ldots, p$.

Standardization is a typical preprocessing step in fitting lasso models since: (1) it ensures that the penalty is applied uniformly across features with different scales of measurement; (2) it often contributes to faster convergence of the optimization algorithm; (3) as we will see in following sections, it simplifies feature screening rules and thus reduces computation complexity.

The rest of the paper is organized as follows. Section \ref{sect_related} reviews the two categories, strong rules and safe rules, upon which our work is built. We propose our new hybrid screening strategy in Section \ref{sect_HSSR} and describe two powerful rules, SSR-BEDPP and SSR-Dome, based on this strategy along with a pathwise coordinate descent algorithm to take advantage of them.  In addition, this section analyzes the computational complexity of the HSSR rules and compares them to SSR and EDPP.  In Section \ref{sect_ext}, we extend SSR-BEDPP to the elastic net and group lasso problems. Section \ref{sect_exp} compares the performance of our rules with existing ones via extensive numerical experiments on synthetic and real data sets for both the standard lasso and the group lasso problems and conclude with some final remarks in Section \ref{sect_conc}.  Proofs of theorems are given in the Appendix.

\section{Existing lasso screening rules}
\label{sect_related}

\subsection{Sequential strong rules} \label{subsect_ssr}

SSR \citep{tibshirani2012strong} is a heuristic screening rule for discarding features when solving the lasso over a grid of decreasing regularization parameter values $\lambda_1 > \lambda_2 > \ldots > \lambda_K$. Specifically, after solving for $\widehat{\bb}(\lambda_{k})$ at $\lambda_{k}$, SSR discards the $j$th feature from the optimization at $\lambda_{k+1}$ if
\begin{align} \label{SSR}
\left| \x_j^T \r(\lambda_k) / n \right| < 2 \lambda_{k+1} - \lambda_{k},
\end{align}
where $\r(\lambda_k) = \y - \X\widehat{\bb}(\lambda_{k})$ is the residual vector at $\lambda_k$. 

To see the rationale of SSR, we start by noting that $\widehat{\bb}(\lambda)$ satisfies the following KKT conditions for the lasso problem (\ref{linear_lasso}):
\begin{align} \label{KKT_lasso}
\begin{cases}
\x_j^T \r(\lambda) / n = \lambda \text{sign}(\widehat{\beta}_j), &\text{ if } \widehat{\beta}_j \neq 0,\\
\left| \x_j^T \r(\lambda) / n \right| \leq \lambda, &\text{ if } \widehat{\beta}_j = 0.
\end{cases}
\end{align}
Let $c_j(\lambda) = \frac{1}{n}\x_j^T \r(\lambda_k)$. The key idea behind SSR is to assume $c_j(\lambda)$ is non-expansive in $\lambda$ (or the ``unit-slope'' bound):
\begin{align} \label{SSR_nonexpan}
\left| c_j(\lambda) - c_j(\tilde{\lambda}) \right| \leq |\lambda - \tilde{\lambda}|, \text{ for any } \lambda, \tilde{\lambda} \in (0, \lambda_{max}].
\end{align}
Now, given $\widehat{\bb}(\lambda_{k}), \lambda_{k}, \lambda_{k+1} \;(\lambda_{k} \geq \lambda_{k+1})$, if conditions (\ref{SSR}) and (\ref{SSR_nonexpan}) are satisfied, we have
\begin{align*}
\left| c_j(\lambda_{k+1}) \right| &\leq \left|c_j(\lambda_{k+1}) - c_j(\lambda_{k}) \right| + \left|c_j(\lambda_{k}) \right| \\
&<\lambda_{k} - \lambda_{k+1} + (2\lambda_{k+1} - \lambda_{k}) \\
& =\lambda_{k+1}, 
\end{align*}
and thus $\widehat{\beta}_j(\lambda_{k+1}) = 0$, implied by the KKT conditions (\ref{KKT_lasso}). 

SSR is simple and able to screen out a large amount of inactive features (i.e., those whose coefficients equal zero). However, since assumption (\ref{SSR_nonexpan}) may be violated, SSR requires checking KKT conditions (\ref{KKT_lasso}) for all $p$ coefficients after convergence has been reached at each value of $\lambda$ to ensure that the calculated $\widehat{\beta}(\lambda_{k+1})$ is a solution to the original optimization problem.  This process is time-consuming when $p$ is large, and even more so if any violations occur, as this involves re-solving the lasso problem with the erroneously discarded features now included.  Fortunately, empirical studies show that violations are rare, although certainly possible; see Section 3 of \citep{tibshirani2012strong} for a thorough analysis.

\subsection{Safe rules} \label{sect_safe}

As noted in the introduction, there are a number of safe rules in the literature; we focus primarily on EDPP rules, as they appear to be the most powerful safe rules developed thus far.  EDPP rules are constructed by projecting the scaled response vector onto a nonempty closed and convex polytope.
Here we derive simplified versions of the basic EDPP rule (BEDPP) and the sequential EDPP rule (SEDPP) under the standardization condition~(\ref{stand}). Compared to original rules, the simplified ones reveal a clearer picture of the computational complexity and reduce the computational burden somewhat. We refer readers to \cite{JMLR:v16:wang15a} for the original EDPP rules and additional technical details. 

The EDPP rules are based on the dual formulation of Problem~\eqref{linear_lasso}:
\begin{align}
\widehat{\bt}(\lambda) = & \argmax_{\bt \in \mathbb{R}^n} \frac{1}{2n} \| \y \|^2 - \frac{n\lambda^2}{2} \|\bt  - \frac{\y}{n\lambda} \|^2 \label{linear_dual} \\
\text{subject to } & |\x_j^T \bt| \leq 1, \quad \forall j=1, \cdots, p, \label{dual_constraints}
\end{align}
where $\widehat{\bt}(\lambda)$ is the dual optimal solution of Problem~\eqref{linear_lasso} under the constraints~\eqref{dual_constraints}. The dual and primal solutions are related via: 
\begin{align} \label{link}
  \widehat{\bt}(\lambda) = \frac{\y - \X \widehat{\bb}(\lambda)}{n\lambda}
\end{align}

The original EDPP rules are developed by exploiting the geometric properties of the dual solutions. The simplified BEDPP and SEDPP rules are stated as the following theorems.

\begin{theorem}[BEDPP] \label{theorem_bedpp}
  For the lasso problem (\ref{linear_lasso}), let $\lambda_m := \lam_{\max} = \max_{j} |\x_j^T \y / n|$ and $\x_* = \argmax_{\x_j} |\x_j^T \y|$. For any $\lambda \in (0, \lambda_{m}]$, under condition (\ref{stand}) we have $\widehat{\beta}_j(\lambda)=0$ if
\begin{multline}
\label{BEDPP_rule}
  \left |(\lambda_m + \lambda) \x_j^T \y - (\lambda_m - \lambda)\text{sign}(\x_*^T \y) \lambda_m \x_j^T \x_* \right | <\\
  2n \lambda \lambda_m - (\lambda_m - \lambda) \sqrt{n \|\y\|^2 - n^2 \lambda_m^2}.
\end{multline}
\end{theorem}

\begin{theorem}[SEDPP] \label{theorem_sedpp}
For the lasso problem (\ref{linear_lasso}), let $\lambda_{m}:= \lam_{\max} = \max_{j} | \x_j^T \y / n|$. Suppose we are given a sequence of $\lambda$ values $\lambda_m = \lambda_0 > \lambda_1 > \ldots > \lambda_K$. Then under condition (\ref{stand}): 
\begin{enumerate}
\item For any $0 < k < K$, we have $\widehat{\beta}_j(\lambda_{k+1})=0$ if  $\widehat{\bb}(\lambda_k)$ is known and the following holds:
\begin{multline}
\label{SEDPP_rule}
\left | \frac{\x_j^T \left(\y - \X \widehat{\bb}(\lambda_k) \right)}{\lambda_k} + \frac{c}{2} \left(\x_j^T \y - \frac{ a \x_j^T \X \widehat{\bb}(\lambda_k)}{ \| \X \widehat{\bb}(\lambda_k) \|^2 } \right) \right | <\\ n- \frac{c}{2}  \sqrt{ n \| \y \|^2 - \frac{n a^2}{\| \X \widehat{\bb}(\lambda_k) \|^2 }}
\end{multline}
where $c = \frac{\lambda_k - \lambda_{k+1}}{ \lambda_k \lambda_{k+1}}$ and $a = \y^T \X \widehat{\bb}(\lambda_k)$ are two scalars.
\item For k = 0, i.e., $\lambda_k = \lambda_m$, SEDPP rule reduces to BEDPP rule. That is, we have $\widehat{\beta}_j(\lambda_{k+1})=0$ if rule (\ref{BEDPP_rule}) holds, in which $(\lambda_m, \lambda)$ is replaced by $(\lambda_0, \lambda_{1})$. 
\end{enumerate}
\end{theorem}


Compared to SEDPP, the BEDPP rule is non-sequential in that screening at $\lambda_{k+1}$ via BEDPP doesn't require the lasso solution at $\lambda_k$. As a result, BEDPP is much simpler to compute but less powerful in discarding inactive features, as shall seen in Section~\ref{subsect_complex}.

%

An alternative safe rule, the Dome test, is similar to BEDPP in that it is non-sequential and requires only a small computational burden; due to space constraints, we omit the details of the Dome test from this paper and refer interested readers to~\cite{xiang2012fast} and \cite{Xiang2016lassoScreen}. A supplementary material containing the details of the simplified Dome test can be found on the GitHub page \footnote{\url{https://github.com/YaohuiZeng/HSSR_paper_supplementary/blob/master/HSSR_supplementary_for_Dome.pdf}}.

\section{Hybrid safe-strong rules} \label{sect_HSSR}

In this section, we define our newly proposed hybrid safe-strong rules (HSSR) and compare their computational complexity to the rules discussed in Section~\ref{sect_related}.  In addition, we present a re-designed pathwise coordinate descent algorithm that takes advantage of these rules to increase the efficiency of solving the lasso.

\subsection{Definition}

 
The motivation of HSSR is to remove a large amount of unnecessary post-convergence KKT checking, required by SSR, on features that could have been discarded by a safe screening rule. In principle, any safe rule can be combined with SSR, resulting in a family of rules which we call hybrid safe-strong rules and define as follows.

\begin{definition}
For solving the lasso problem~(\ref{linear_lasso}) over a sequence of $\lambda$ values $\lambda_1 > \lambda_2 > \ldots > \lambda_K$, suppose that there exists a safe rule and that $\widehat{\bb}(\lambda_k)$ is known. Let $\mathcal{S}_{k+1}$ denote the {\em safe set}, i.e., the set of features not discarded by the safe rule at $\lambda_{k+1}$. Then a corresponding hybrid safe-strong rule (HSSR) can be formulated by combining the safe rule with SSR. Specifically, HSSR discards the $j$th feature from the lasso optimization at $\lambda_{k+1}$ if
\begin{align} \label{hssr_rule}
j \in \mathcal{S}_{k+1}^c \cup \{j \in \mathcal{S}_{k+1}: |\x_j^T\r(\lambda_k)| / n| < 2 \lambda_{k+1} - \lambda_k \}, 
\end{align}
where $\r(\lambda_k) = \y - \X \widehat{\bb}(\lambda_k)$.
\end{definition}

HSSR builds upon SSR and thus enjoys all of its advantages: simple, sequential, and powerful to discard a large portion of features. As a drawback, it also requires post-convergence KKT checking. However, HSSR only needs to perform KKT checking over a subset of features since all features in the set $\mathcal{S}_{k+1}^c$ are discarded by the safe rule. Provided that the safe rule is simple to calculate, by which we mean that its time complexity is $O(np)$, HSSR will be more efficient computationally than SSR.
The amount of efficiency gained depends on the safe rule, with more powerful rules providing greater increases in speed.

In this paper, two instances of HSSR, namely SSR-BEDPP and SSR-Dome, are studied. These two rules respectively use BEDPP and the Dome test as the safe rule.
An essential property of HSSR is that for any problem with a unique global optimum and algorithm that converges to that solution, incorporating HSSR into the algorithm will yield the same solution, as stated in the following theorem.


\begin{theorem}[Convergence]
\label{theorem_convergence}
Suppose the lasso problem~\eqref{linear_lasso} at a given $\lambda$ is strictly convex such that the sequence of solutions produced by an iterative algorithm $a(\cdot)$ (such as coordinate descent) converges to the  unique global optimum, $\widehat{\bb}(\lambda)$. Then that algorithm with HSSR screening converges to the same solution $\widehat{\bb}(\lambda)$. 
\end{theorem}

\begin{proof}
Let $\X_{\mathcal{S}}$ denote the submatrix of $\X$ consisting only of the features in $\mathcal{S}(\lambda)$. By the definition of a safe rule, the global optimum $\widehat{\bb}(\lambda)$ can be decomposed as $\widehat{\bb}(\lambda) = \left(\zero, \widehat{\bb}^T_{\mathcal{S}}(\lambda) \right)^T$, where $\widehat{\bb}_{\mathcal{S}}(\lambda)$ is the solution to the following optimization problem:
\begin{align} \label{reduced_opt}
\widehat{\bb}_{\mathcal{S}}(\lambda) = \argmin_{\bb_\mathcal{S} \in \mathbb{R}^{|\mathcal{S}(\lambda)|}} \frac{1}{2n} \|\y - \X_{\mathcal{S}} \bb_\mathcal{S}\|^2 + \lambda \|\bb_\mathcal{S}\|_1.
\end{align}

Furthermore, it's easy to verify that the algorithm $a(\cdot)$ with SSR screening for solving~\eqref{reduced_opt} converges to the global optimum $\widehat{\bb}_{\mathcal{S}}(\lambda)$. This is because the KKT checking procedure required by SSR guarantees the final solution satisfies the KKT optimality conditions and hence is the global optimum. Therefore, the algorithm with HSSR screening converges to $\widehat{\bb}(\lambda)$. 
\end{proof}


\subsection{Performance analysis} \label{subsect_complex}

Intuitively, the computational savings achieved by feature screening will be negated if the screening rule itself is too complicated to execute. Therefore, an efficient rule needs to balance the trade-off between its computational complexity and rejection power (i.e., how many features can be discarded). That is, an ideal screening rule should be powerful enough to discard a large portion of features and also relatively simple to compute.
To show the advantages of HSSR, we compare the aforementioned screening rules in terms of the rejection power and computational complexity of the rules themselves.

\subsubsection{Screening power}

Here we present an empirical comparison of different rules in terms of the power to discard features. Figure~\ref{fig_reject} depicts the results based on the GENE data (See details in Section~\ref{subsect_real_lasso}). First, it's important to note that HSSR, by construction, discards at least as many features as SSR does. Second, HSSR, SSR and SEDPP discard far more features than the non-sequential rules BEDPP and Dome.
In particular, the screening power of BEDPP and Dome decreases rapidly as $\lambda$ decreases. For example, BEDPP cannot discard any features when $\lambda / \lambda_{\max}$ is smaller than 0.45 in this case, whereas Dome is the least powerful and discards virtually no features when $\lambda / \lambda_{\max}$ is less than 0.6.

\begin{figure}
\centering
\includegraphics[width=0.65\linewidth]{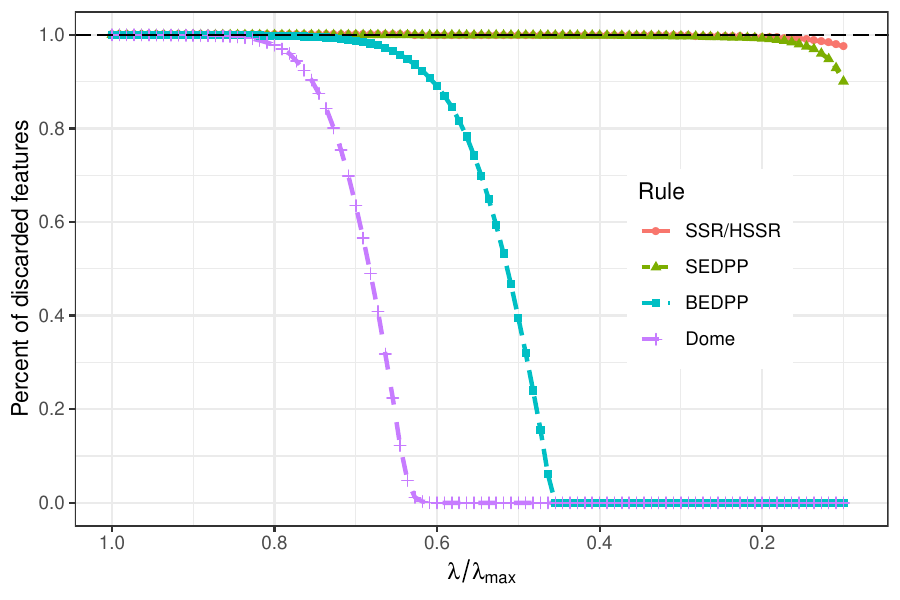}
\caption{Percent of features discarded.}
\label{fig_reject}
\end{figure}

\subsubsection{Computational complexity}

Table~\ref{tab_complexity} presents the complexity of computing these rules for the entire path of $K$ values of $\lambda$.

For SSR~(\ref{SSR}), it's important to observe that the quantities needed to check the KKT conditions~(\ref{KKT_lasso}), $\x_j^T \r(\lambda_k)$, can be re-used for executing SSR at $\lambda_{k+1}$ for that feature. Therefore, SSR requires $O(np)$ operations, as the dominant computation is calculating $\X^T\r(\lambda_k)$. However, since $\r(\lambda_k)$ changes as a function of $\lambda_k$, the total complexity of SSR is $O(npK)$ over the entire solution path. 

HSSR, on the other hand, only needs to perform KKT checking over the features not discarded by the safe screening step.  Thus, $\x_j^T \r(\lambda_{k-1})$ must be calculated only for features in the safe set $\mathcal{S}_{k}$, yielding $O(n\sum_{k=1}^K |\mathcal{S}_k|))$ operations. When the safe rule is effective (e.g. when $\lambda$ is relatively large, as shown in Figure~\ref{fig_reject}), HSSR would avoid a large amount of unnecessary KKT checking and hence be much more efficient than SSR.

The complexity of SEDPP~(\ref{SEDPP_rule}) is more involved. During coordinate descent, the residuals $\r(\lambda_k)$ are continually updated and stored.  Thus, $\X\widehat{\bb}(\lambda_{k})$ can be obtained at a cost of $O(n)$ operations since $\X\widehat{\bb}(\lambda_{k}) = \y - \r(\lambda_k)$. Furthermore, only $O(n)$ calculations are needed to update $\|\X\widehat{\bb}(\lambda_{k})\|$ and $a$, while quantities like $\x_j^T\y$ and $\| \y \|$ can be pre-computed to avoid duplicated calculations. The more demanding parts are on the left hand side of~(\ref{SEDPP_rule}), specifically, the two terms $\x_j^T\r(\lambda_k)$ and $\x_j^T\X\widehat{\bb}(\lambda_{k})$.  Since these must be calculated for all features, this essentially involves calculating $\X^T\r(\lambda_k)$ and $\X^T\X\widehat{\bb}(\lambda_{k})$, both of which require $O(np)$ calculations.  Thus, similar to that of SSR, the total complexity of SEDPP is $O(npK)$ for obtaining the entire solution path. 

Finally, the complexity of executing BEDPP~(\ref{BEDPP_rule}) over the solution path is only $O(np)$ as its dominant calculations are $\X^T\y$ and $\X^T\x_*$, which only need to be calculated once. After these initial calculations, only $O(p)$ operations are needed to compute the rule, resulting in a complexity of $O(pK)$ over the entire path. Hence the total complexity is $O(np)$ provided that $n$ is larger than $K$. The Dome test also has complexity of $O(np)$, and can be analyzed in the same fashion based on results in~\cite{xiang2012fast}.

\begin{table}[t]
\caption{Complexity of computing screening rules along the entire path of $K$ values of $\lambda$. $| \mathcal{S}_k |$ is the cardinality of safe set $\mathcal{S}_k$ obtained by the safe rule used in HSSR.}
\centering
\begin{tabular}{cccccc}
\toprule
Dome & BEDPP & SEDPP & SSR & HSSR \\
\midrule
$O(np)$ & $O(np)$ & $O(npK)$ & $O(npK)$ & $O(n\sum_{k=1}^K |\mathcal{S}_k|))$\\
\bottomrule
\end{tabular}
\label{tab_complexity}
\end{table}

\subsubsection{Advantages of HSSR}
The advantages of HSSR can be summarized as follows:
\begin{enumerate}
\item \textbf{Computational efficiency}: Solving the lasso with HSSR screening, as compared to other rules, involves the least computational burden. As we will see in Section~\ref{sect_exp}, the result is that HSSR is the fastest of the approaches considered here.
\item \textbf{Memory efficiency}: Both SSR and SEDPP have to fully scan the feature matrix $K$ times, while HSSR only needs to do so for the portion of the lasso path where the safe rule is not able to discard any features. HSSR is therefore more memory-efficient, a particularly appealing advantage in out-of-core computing, where fully scanning the feature matrix requires disk access and therefore becomes the computational bottleneck.
\item \textbf{Generalizability}: HSSR is a rather general lasso screening framework, and can be easily extended to other lasso-type problems such as the elastic net and the group lasso.
\end{enumerate}

\subsection{Pathwise coordinate descent with HSSR}


The pathwise coordinate descent (PCD) algorithm~\citep{friedman2007pathwise} solves the lasso solution path along a grid of decreasing parameter values $\lambda_1 > \lambda_2 > \ldots > \lambda_K$. When solving for $\widehat{\bb}(\lambda_k)$, PCD utilizes previous solution $\widehat{\bb}(\lambda_{k-1})$ as warm starts. This ``warm start'' strategy makes the algorithm very efficient. 

In this section, we re-design the PCD algorithm by incorporating HSSR, as described in Algorithm~\ref{algo_PCD_HSSR}. The algorithm starts by initializing the safe sets $\mathcal{S}$ and $\mathcal{S}_{prev}$, which saves the safe set at previous iteration. Another set $\mathcal{H}$, called the strong set, is also initialized to store the features in the safe set not discarded by SSR screening. The \texttt{Flag} variable indicates whether the safe rule screening should be turned off or not. The rationale of this design is to stop using the safe rule once it is no longer capable of discarding any features (See Figure~\ref{fig_reject}). Note also that the algorithm only needs to update $z_j$ for those ``newly-entered'' features in the safe set (line 4) before conducting SSR screening. This is because all $z_j$'s associated with features in $\mathcal{S}$ must have already been computed during post-convergence KKT checking at the previous $\lambda$ (line 15).

\begin{algorithm}[ht]
    \SetKwInOut{Input}{Input}
    \SetKwInOut{Output}{Output}
    \SetKwInOut{Return}{Return}
	\SetKwInOut{Initialize}{Initialize}
	\SetKw{KwGoTo}{go to}
	
    \Input{$\{\x_j\}_{j=1}^p$, $\y$, $\lambda_{max} = \lambda_0 > \lambda_1 > \ldots > \lambda_K$}
    \Initialize{$\mathcal{S} = \mathcal{S}_{prev} = \emptyset, ~ \mathcal{H} = \emptyset, ~ \r = \y, ~ \{z_j = 0: j = 1, 2, \ldots, p\}$, Flag = FALSE} 
	\For{$k \leftarrow 1$ \KwTo $K$} {

		\If{Flag = FALSE} {
        		Safe Screening: $\mathcal{S} \leftarrow \{j: \x_j \text{ survives safe screening}\}$
        		
        		Update $z_j = \x_j^T \r/n$ over set $\{j: j \in \mathcal{S} ~ \backslash ~ \mathcal{S}_{prev}\}$
        		
        		
        		$\mathcal{S}_{prev} \leftarrow \mathcal{S}$
        		
    			\If{$|\mathcal{S}| = p$} {
        			Flag $\leftarrow$ TRUE
        			
        		} 
      	}

		SSR screening: $\mathcal{H} \leftarrow \{j \in \mathcal{S}: |z_j | \geq 2 \lambda_{k} - \lambda_{k-1} \}$
		
		\While{not converged}{\label{solve_lasso}
		
			Solve~(\ref{linear_lasso}) for $\widehat{\bb}(\lambda_k)$ via coordinate descent iteration over features only in $\mathcal{H}$ and keep updating $\r$ 
			
		}

		Update $z_j = \x_j^T \r/n$ over set $\{j: j \in \mathcal{S} ~ \backslash ~ \mathcal{H}\}$, and check KKT violations: $\mathcal{V} \leftarrow \{j \in \mathcal{S} ~ \backslash ~ \mathcal{H}: | z_j | \geq \lambda_k \}$
			
		\If{$\mathcal{V} \neq \emptyset$}{
		
			$\mathcal{H} \leftarrow \mathcal{H} \cup \mathcal{V}$
				
			\KwGoTo \ref{solve_lasso} with current solution as a warm start
		}
				
		\textbf{save} $\widehat{\bb}_k$	

	}
	
    \Output{$\{\widehat{\bb}\}_{k=1}^K$}
    \caption{PCD algorithm with HSSR screening}
    \label{algo_PCD_HSSR}
\end{algorithm}

After SSR screening, the algorithm then solves the lasso problem for $\widehat{\bb}(\lambda_k)$ via coordinate descent iterations using features only in the strong set $\mathcal{H}$, as described by the \texttt{while} loop, until a predefined convergence criterion is met.

The post-convergence KKT checking takes place in line 15 after a solution is obtained: KKT checking is applied to features that are outside of the strong set $\mathcal{H}$ but in the safe set $\mathcal{S}$. If any violations are detected, the strong set is updated by adding in the features which violate the KKT conditions, and the lasso then needs to be re-solved (line 18) with the updated strong set.

\section{Extensions to other lasso-type problems} \label{sect_ext}


\subsection{SSR-BEDPP for the elastic net}

The elastic net estimator $\widehat{\bb}(\lambda, \alpha)$ is defined~\citep{zou2005regularization} as the argument minimizing
\begin{align} \label{linear_enet}
\frac{1}{2n} \|\y - \X \bb\|^2 + \alpha \lambda \|\bb\|_1 + \frac{(1 - \alpha) \lambda}{2} \| \bb \|^2.
\end{align}

SSR can be applied to the elastic net with minimal changes, as shown in \cite{tibshirani2012strong}. Specifically, SSR discards the $j$th feature from the elastic net optimization at $\lambda_{k+1}$ if
\begin{align} \label{SSR_enet}
\left| \x_j^T \r(\lambda_k) / n \right| < \alpha (2 \lambda_{k+1} - \lambda_{k}).
\end{align}

Moreover, it can be shown that the KKT conditions for~(\ref{SSR_enet}) are
\begin{alignat}{2} \label{KKT_enet}
\x_j^T \r(\lambda) / n - (1-\alpha) \lambda \widehat{\beta}_j &= \lambda \text{sign}(\widehat{\beta}_j) & &\quad\text{ if } \widehat{\beta}_j \neq 0,\\
\left| \x_j^T \r(\lambda) / n - (1-\alpha) \lambda \widehat{\beta}_j \right| &\leq \lambda & &\quad\text{ if } \widehat{\beta}_j = 0.
\end{alignat}

The BEDPP rule in~\cite{{JMLR:v16:wang15a}} is not directly applicable to the elastic net problem. Here we extend BEDPP to the elastic net as the following theorem. 

\begin{theorem} [BEDPP for elastic net]
  \label{theorem_bedpp_elastic_net}
  For the elastic net problem (\ref{linear_enet}), let $\lambda_m := \lambda_{max} = \max_{j} | \frac{\x_j^T \y}{\alpha n} |$ and $\x_* = \argmax_{\x_j} |\x_j^T \y|$. Under condition (\ref{stand}), for any $\lambda \in (0, \lambda_{m}]$ and $\x_j \neq \x_*$, we have $\widehat{\beta}_j(\lambda)=0$  if
\begin{multline}
\label{BEDPP_enet}
\left|(\lambda_m + \lambda) \x_j^T \y - (\lambda_m - \lambda) \frac{\text{sign}(\x_*^T \y) \alpha \lambda_m}{1 + \lambda (1 - \alpha)} \x_j^T \x_* \right| <\\ 2 n \alpha \lambda \lambda_m - (\lambda_m - \lambda) \sqrt{ n\| \y \|^2 (1 + \lambda (1 - \alpha)) - n^2 \alpha^2 \lambda_m^2}
\end{multline}
\end{theorem}


Analogous to \eqref{BEDPP_rule}, the complexity of \eqref{BEDPP_enet} for solving the elastic net over an entire solution path is $O(np)$ since, again, $O(np)$ calculations are needed to pre-compute quantities $\X^T\y$, $\X^T \x_*$, and $\| \y \|$. After that, only $O(p)$ operations are required to execute the rule. Moreover, given \eqref{SSR_enet}, \eqref{KKT_enet}, and \eqref{BEDPP_enet}, Algorithm~\ref{algo_PCD_HSSR} may be used for the elastic net, with appropriate modifications to the screening rules, KKT checking, and coordinate descent update.

\subsection{SSR-BEDPP for the group lasso}

As another example, we extend SSR-BEDPP to the group lasso problem. Suppose we have $p$ features assigned into $G$ non-overlapping groups. Let $W_g$ denote the number of features in the $g$th group. The group lasso problem \citep{Yuan2006} is defined as
\begin{align} \label{linear_glasso}
\widehat{\bb}(\lambda) = \argmin_{\bb \in \mathbb{R}^p} \frac{1}{2n} \left \| \y - \sum_{g=1}^G \X_g \bb_g \right \|^2+ \lambda \sum_g \sqrt{W_g} \|\bb_g\|,
\end{align}
where $\bb = (\bb_1^T, \ldots, \bb_G^T)^T$, $\X_g$ is the $n \times W_g$ sub-matrix whose columns correspond to features in group $g$, and $\bb_g = (\beta_{g,1}, \ldots, \beta_{g,W_g})^T$ is the associated coefficient vector. Here, in addition to the standardization described in Section~\ref{sect_intro}, we apply an additional level of standardization at the group level~\citep{Breheny2015}:
\begin{align} \label{stand_glasso}
\frac{1}{n} \X_g^T \X_g= \I, \quad g = 1, \ldots, G.
\end{align}

Given $\widehat{\bb}(\lambda_{k})$, it can be shown \citep{tibshirani2012strong} that SSR discards the $g$th group of coefficient vector $\widehat{\bb}_g(\lambda_{k+1})$ from the group lasso optimization at $\lambda_{k+1}$ if 
\begin{align} \label{SSR_glasso}
\left \| \frac{1}{n} \X^T_g \r(\lambda_k) \right\| < \sqrt{W_g} (2 \lambda_{k+1} - \lambda_k),
\end{align}
where $\r(\lambda_k) = \y - \sum_{\ell=1}^G \X_{\ell} \widehat{\bb}_{\ell} (\lambda_k)$. Moreover, the KKT conditions for~(\ref{linear_glasso}) are,
\begin{align} \label{KKT_glasso}
 \X_g^T \r(\lambda) / n = \lambda \sqrt{W_g} \bt_g, \quad g = 1, \ldots, G,
\end{align}
where $\bt_g$ is a subgradient of $\| \widehat{\bb}_g \|$.

EDPP rules have also been derived for the group lasso \citep{JMLR:v16:wang15a}. With some algebra, we present a simplified BEDPP under condition~(\ref{stand_glasso}) for the group lasso as the following theorem.
\begin{theorem} [BEDPP for group lasso]
\label{theorem_BEDPP_glasso}
For the group lasso problem (\ref{linear_glasso}), let $\lambda_m := \lambda_{max} = \max_g \frac{\| \X_g^T \y \|}{ n \sqrt{W_g}}$, $g_* = \argmax_{g} \frac{\| \X_g^T \y \|}{ n \sqrt{W_g}}$, $\X_*$ and $W_*$ the data matrix and size of the group associated with $g_*$, and $\bar{\v} = \X_* \X_*^T \y$. For any $\lambda \in (0, \lambda_m]$ and $g = 1, 2, \ldots, G$, under condition (\ref{stand_glasso}) we have $\widehat{\bb}_g(\lambda) = \zero$ if
\begin{multline}
\label{BEDPP_glasso}
\sqrt{A-B+C} <\\ 2 n \lambda \lambda_m \sqrt{W_g} - (\lambda_m -\lambda) \sqrt{n \| \y \|^2 - n^2 \lambda_m^2 W_*}
\end{multline}
where
\begin{align*}
  A &= (\lambda + \lambda_m)^2 \| \X_g^T \y\|^2\\
  B &= 2 (\lambda_m^2 - \lambda^2) \y^T \X_g \X_g^T \bar{\v}/n\\
  C &= (\lambda_m - \lambda)^2 \| \X_g^T \bar{\v} \|^2/n^2.
\end{align*}
\end{theorem}

Analogous to the lasso and elastic net, the complexity of executing~(\ref{BEDPP_glasso}) for an entire solution path costs $O(np)$. To see this, note that $\bar{\v}$ only needs to be calculated once and requires $O(nW_*)$ operations.  Thus, the most computationally intensive step of \eqref{BEDPP_glasso} is calculating $\X_g^T\bar{\v}$ and $\y^T\X_g$, each of which require $O(nW_g)$ operations, or $O(np)$ operations to calculate these intermediate quantities for all $G$ groups.  Once this is done, executing BEDPP rule for group lasso costs only $O(p)$ at each $\lambda_k$.

Given SSR and BEDPP rules for the group lasso, we can formulate the SSR-BEDPP rule and solve the group lasso based on Algorithm~\ref{algo_PCD_HSSR} with appropriate modifications to the screening rules and KKT checking given by \eqref{SSR_glasso}, \eqref{BEDPP_glasso}, and \eqref{KKT_glasso}. The coordinate descent update must also be modified to a group descent update (also known as a blockwise coordinate descent update) as described in \cite{qin2013efficient, Breheny2015, meier2008group}.

\section{Experiments} \label{sect_exp}

In this section, we conduct experiments to show that our proposed hybrid safe-strong rules significantly outperform the existing SSR and SEDPP rules. We also take into comparison the ``Active-set Cycling'' (AC) strategy~\citep{lee2007efficient}. AC is somewhat similar to SSR in that they both begin by solving the lasso over a subset of features and then check KKT conditions to verify the solution. AC, however, merely cycles over the nonzero coefficients. The idea of AC is simple and effective, and has been commonly applied to large-scale sparse learning problems with considerable speedup observed \citep{garrigues2009homotopy, tibshirani2012strong, lee2015strong, meier2008group}.

In all numerical experiments, we focus on solving the lasso or the group lasso problems over the entire path of 100 values of $\lambda$  which are equally spaced on the scale of $\lambda / \lambda_{max}$ from 0.1 to 1. All experiments in this section are conducted with 20 replications, and the average computing times (in seconds) are reported. The benchmarking platform is a MacBook Pro with Intel Core i7 @ 2.3 GHz and 16 GB RAM.

In every experiment, all algorithms converged to the same solutions at all values of $\lam$, to within numerical tolerance.  This was measured by the relative difference $RD(\lambda) = \{Q_\lam(\hat{\bb}_B) - Q_\lam(\hat{\bb}_G)\} / Q_\lam(\hat{\bb}_G)$, where $Q(\bb)$ denotes the objective (loss + penalty), $\hat{\bb}_B$ denotes the solution using the algorithm and implementation described here (using the \texttt{biglasso} package) and $\hat{\bb}_G$ is the solution based on a reference implementation (the \texttt{glmnet} package); in all situations, $\abs{RD(\lam)} < 2 \times 10^{-5}$.

\subsection{Results for the lasso}

In this section, we compare SSR-BEDPP and SSR-Dome with existing methods AC, SSR, and SEDPP in solving the standard lasso problem.  Basic pathwise coordinate descent (``Basic PCD'') with no screening or active cycling is used as baseline for the comparison. Our R package \texttt{biglasso}
(Version 1.3-2) implements all these methods and is used for all the numerical studies.

\subsubsection{Synthetic data} \label{subsect_sim_lasso}

We first demonstrate with synthetic data that SSR-BEDPP is more scalable in both $n$ and $p$ (i.e., number of observations and features). We adopt the same model in \cite{JMLR:v16:wang15a} to simulate data: $\y = \X \bb + 0.1 \bep$, where $\X$ and $\bep$ are i.i.d. sampled from $N(0, 1)$. Here we consider two different cases: 
\begin{enumerate}[(a)]
\item \textbf{Case 1: varying $p$}. We set $n=1,000$ and vary $p$ from 1,000 to 100,000. We randomly select 20 true features, and sample their coefficients from Unif[-1, 1]. After simulating $\X$ and $\bb$, we then generate $\y$ according to the true model;
\item \textbf{Case 2: varying $n$}. We set $p=10,000$ and vary $n$ from 200 to 10,000. $\bb$ and $\y$ are generated in the same way as in Case 1.
\end{enumerate}

Figure~\ref{fig_sim_lasso} compares the average computing time of solving the lasso over a sequence of 100 values of $\lambda$ by the different methods. In all settings, our rule SSR-BEDPP is uniformly 5x faster than Basic PCD. More importantly, SSR-BEDPP is around 2x faster than state-of-the-art rules SSR and SEDPP. Note that the computing times of SSR and SEDPP are almost the same so the lines of these two cannot be distinguished in the plots. Perhaps most surprisingly, SSR and SEDPP provide only a small advantage over AC, while SSR-BEDPP achieves more than a 2x additional speedup compared to AC, suggesting that HSSR does not merely add together the efficiency gains of SSR and BEDPP, but accomplishes something novel by hybridizing them together.

It's worth mentioning that the new rule SSR-Dome can also provide substantial speedup - 1.6x faster than AC and 1.4x faster than SSR or SEDPP, demonstrating the effectiveness of hybrid screening as a general optimization strategy. Since the Dome test itself is less powerful than the BEDPP rule \citep{JMLR:v16:wang15a} and takes equally long to compute, it is not surprising that SSR-BEDPP is the faster of the two approaches.

The simulations in this section involve independent features; the effect of correlation on the proposed algorithm is explored in the next section as well as in Appendix~\ref{App:Cor}.

\begin{figure}[t]
\centering
\includegraphics[width=\linewidth]{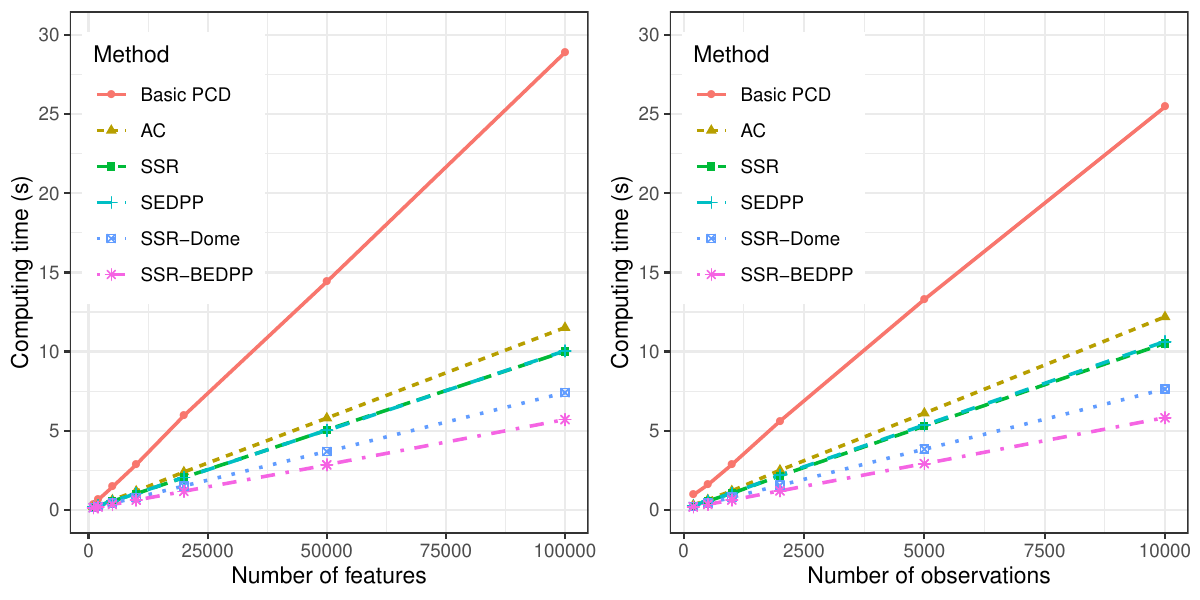}
\caption{Average computing time as a function of $p$ (left) and $n$ (right) for solving the lasso along a sequence of 100 values of $\lambda$. Note that the lines for SSR and SEDPP overlap and cannot be distinguished.}
\label{fig_sim_lasso}
\end{figure}

\subsubsection{Real data} \label{subsect_real_lasso}

Real-world data sets often have complicated signals and correlation structures which affect the performance of the screening rules. In this section, we compare the aforementioned methods using diverse real data sets:
\begin{enumerate}[(a)]
\item \textbf{Breast cancer gene expression data\footnote{\url{http://myweb.uiowa.edu/pbreheny/data/bcTCGA.html}} (GENE)}: this data set contains gene expression measurements of 17,322 genes of 536 breast cancer patients from The Cancer Genome Atlas project. The goal is to identify genes with expression levels related to that of the tumor suppressor gene BRCA1.
\item \textbf{MNIST handwritten image data\footnote{\url{http://yann.lecun.com/exdb/mnist/}} (MNIST)}: this data set contains grayscale images of handwritten digits with 60,000 images for training and 10,000 for testing. Each image is of $28 \times 28$ dimension. Following~\cite{JMLR:v16:wang15a}, we first use the training set to construct a feature matrix $\X \in \mathbb{R}^{784 \times 60000}$. We then randomly choose an image in the test set as the response vector $\y \in \mathbb{R}^{784}$ for each of the 20 replications.
\item \textbf{Cardiac fibrosis genome-wide association data\footnote{\url{https://arxiv.org/abs/1607.05636}} (GWAS)}: this data set contains single nucleotide polymorphism (SNP) data collected on 313 human hearts. The goal of the study is to discover SNPs that are associated with increased fibrosis. 
The response vector $\y \in \mathbb{R}^{313}$ is the log of the cardiomyoctye:fibroblast ratio, and the feature matrix $\X \in \mathbb{R}^{313 \times 660,496}$ records the data for the 660,496 SNPs.
\item \textbf{Subset of New York Times bag-of-words data\footnote{\url{https://archive.ics.uci.edu/ml/datasets/Bag+of+Words}} (NYT)}: this data set is from the UCI Machine Learning Repository \citep{Lichman:2013}. The raw data matrix contains 300,000 documents represented as rows of 102,660 words, where the cell $(i, j)$ records the number of occurrences of word $j$ in article $i$. Following~\citet{Xiang2016lassoScreen}, we preprocess the raw data by first removing documents with low word counts and then randomly selecting a subset of 5,000 documents and 55,000 words to form the feature matrix $\X \in \mathbb{R}^{5000 \times 55000}$. At each replication, a word column is randomly chosen from the rest set to be the  response $\y \in \mathbb{R}^{5000}$.
\end{enumerate}

Table \ref{tab_real_res_lasso} summarizes the dimensions of the  real data sets and the average computing times. The speedup of different methods relative to Basic PCD is depicted in Figure~\ref{fig_real_res_lasso}. Again, SSR-BEDPP outperforms all other methods with the most speedup on all data sets, ranging from 13.8x (NYT) to 52.7x (MNIST) faster than Basic PCD. 

In comparison to AC, SSR-BEDPP results in additional speedup ranging from 2.2x on GENE data to 3.7x on MNIST data. SSR and SEDPP, however, provide a meaningful improvement over AC only for the GWAS data; for the other three data sets, the speedup is quite small.  Overall, SSR-BEDPP is 1.3x to 3.2x faster than SSR and SEDPP based on the four real data sets.

\begin{table}[t]
\caption{Average computing time (standard error) for solving the lasso along a sequence of 100 values of $\lambda$ on real data sets.}
\centering
\begin{adjustbox}{max width=\linewidth, keepaspectratio, center}
\begin{tabular}{rcccc}
\toprule
\multirow{3}{*}{Method} & 
	\multicolumn{1}{c}{GENE} & \multicolumn{1}{c}{MNIST} &
	\multicolumn{1}{c}{GWAS} & \multicolumn{1}{c}{NYT} \\
	& \multicolumn{1}{c}{$n=536$} & \multicolumn{1}{c}{$n=784$} &
	\multicolumn{1}{c}{$n=313$} & \multicolumn{1}{c}{$n=5,000$} \\
	& \multicolumn{1}{c}{$p=17,322$} & \multicolumn{1}{c}{$p=60,000$} &
	\multicolumn{1}{c}{$p=660,495$} & \multicolumn{1}{c}{$p=55,000$} \\
\midrule
Basic PCD & 12.84 (0.06) & 91.73 (6.32) & 266.22 (1.14) & 246.87 (24.12) \\
AC & 1.54 (0.01) & 6.48 (0.11) & 43.59 (0.19) & 44.57 (1.96) \\
SSR & 1.13 (0.01) & 5.58 (0.04) & 21.89 (0.10) & 33.64 (0.64) \\
SEDPP & 1.26 (0.02) & 5.57 (0.04) & 21.47 (0.07) & 35.26 (1.21) \\
SSR-Dome & 0.86 (0.01) & 2.92 (0.07)& 18.87 (0.10) & 23.01 (1.59) \\
SSR-BEDPP & \textbf{0.69 (0.01)} & \textbf{1.74 (0.09)} & \textbf{16.27 (0.08)} & \textbf{17.88 (1.75)} \\
\bottomrule
\end{tabular}
\end{adjustbox}
\label{tab_real_res_lasso}
\end{table}

\begin{figure}[t]
\centering
\includegraphics[width=0.7\linewidth]{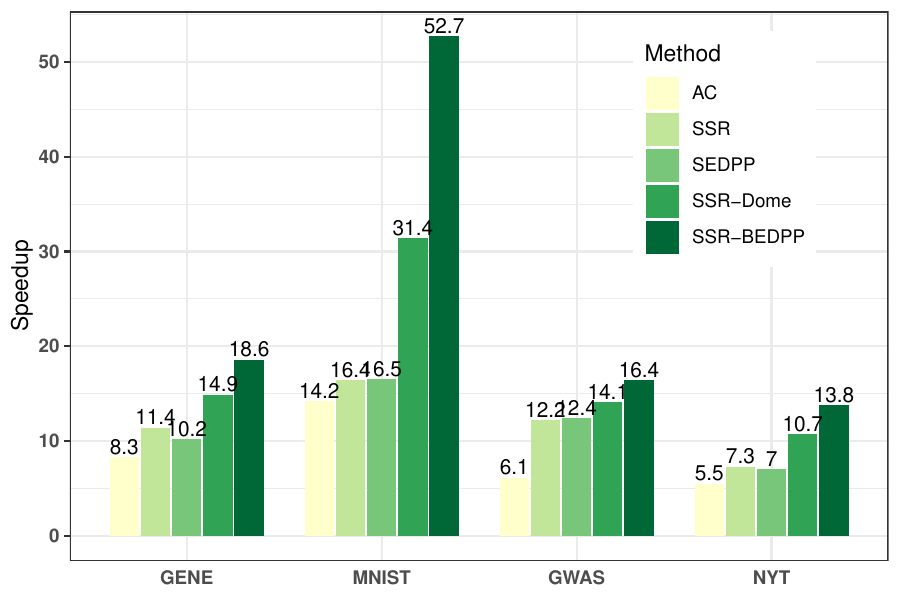}
\caption{The speedup relative to Basic PCD for solving the lasso along a sequence of 100 values of $\lambda$ on real data sets.}
\label{fig_real_res_lasso}
\end{figure}



\subsection{Results for the group lasso}

In this section, we conduct experiments via our R package \texttt{grpreg}\footnote{See Version 3.1-1 at: \url{https://github.com/YaohuiZeng/grpreg}} (Version 3.1-1) to compare SSR-BEDPP with existing methods AC, SSR, and SEDPP in solving the group lasso problem. Note again that basic group descent algorithm (``Basic GD'') with no screening or active cycling is used as baseline.

\subsubsection{Synthetic data}

To generate the synthetic data, we again use the model: $\y = \X \bb + 0.1 \bep$, where $\X$ and $\bep$ are i.i.d. sampled from $N(0, 1)$. Here we fix the number of observations $n$ to be 1,000, and the number of features in all groups to be 10. We vary the number of total groups from 100 to 10,000. In all settings, we randomly select 10 nonzero groups (i.e., groups of features that having nonzero coefficients), and sample the 100 coefficients in these groups from Unif[-1, 1]. After simulating $\X$ and $\bb$, we then obtain $\y$ according to the true model.

Figure~\ref{fig_vary_ngrp} depicts the average computing time of solving the group lasso over a sequence of 100 $\lambda$ values. Again, the computing times by SSR and SEDPP are so close that the corresponding two lines cannot be distinguished. Similar conclusions as for lasso case can be drawn here: (1) our new rule SSR-BEDPP provides remarkable reduction of computing time uniformly across all settings by more than 7x speedup compared to Basic GD, and by around 2x speedup compared to SSR and SEDPP; (2) SSR and SEDPP performs almost identically, and offer only a small advantage over AC.

\begin{figure}[ht]
\centering
\includegraphics[width=0.7\linewidth]{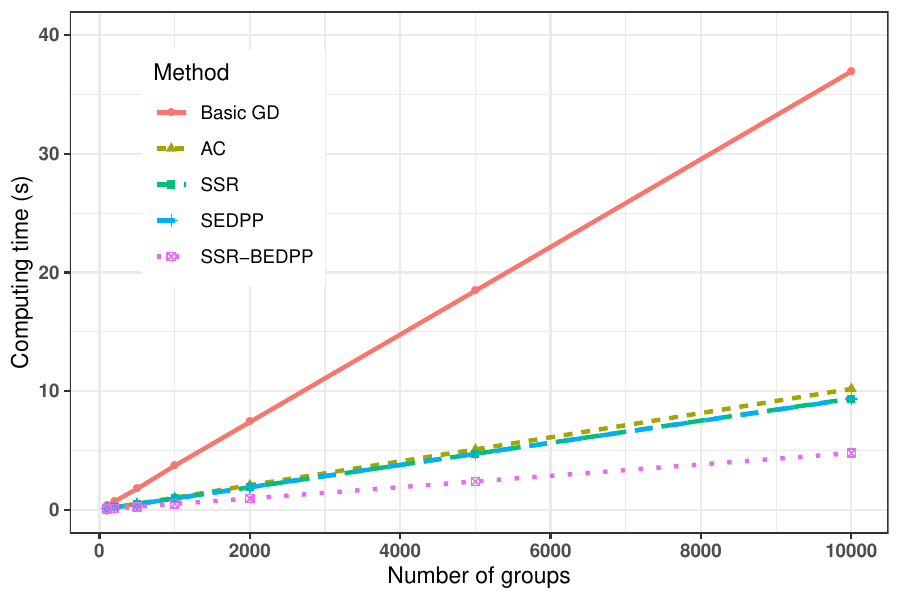}
\caption{Average computing time as a function of the number of groups for solving the group lasso along a sequence of 100 values of $\lambda$. Note that the lines for SSR and SEDPP overlap and cannot be distinguished.}
\label{fig_vary_ngrp}
\end{figure}

\subsubsection{Real data}

We evaluate the performance of different rules using the following real data sets. 
\begin{enumerate}[(a)]
\item \textbf{Genetic rare variant study data (GRVS)}: this data set contains real exon sequencing data from the 1000 Genomes Project \cite{10002010map} on 697 subjects and 24,487 genetic variants. The genetic variants are grouped into 3205 genes (i.e., $n=697$, $p = 24,487$, and $G = 3,205$). 20 different response vectors containing the quantitative phenotypes are simulated according to a plausible genetic model of variant-disease association described in \cite{almasy2011genetic}.
\item \textbf{B-spline regression on GENE data (GENE-SPLINE)}: here we revisit the GENE data in Section~\ref{subsect_real_lasso} and fit a B-spline regression model using the group lasso. Specifically, 5-term basis expansions are first applied to each of the 17,322 features in GENE data, resulting in 86,610 new features in total. The 5 basis terms for each original feature are treated as forming a group. Therefore, $n=536$, $p = 86,610$ and $G = 17,322$.
\end{enumerate}

Table~\ref{tab_real_res_glasso} presents the average computing time and the speedup relative to Basic GD for solving the group lasso along a sequence of 100 values of $\lambda$ on the above two real data sets. SSR-BEDPP again outperforms other methods on the two real data sets with 6.3x and 33.4x speedup compared to Basic GD. In addition, it's around 1.4x faster than SSR and SEDPP, which again show similar performance and only small improvements over active cycling. Finally, SSR-BEDPP is over 1.5x faster than AC for GRVS data, and nearly 2x faster for GENE-SPLINE data.

\begin{table}[t] 
\caption{Average computing time (standard error) and the speedup relative to Basic GD for solving group lasso along a sequence of 100 values of $\lambda$ on real data sets.}
\centering
\begin{adjustbox}{max width=\linewidth, keepaspectratio, center}
\begin{tabular}{rcccc}
\toprule
\multirow{1}{*}{Method} & 
	\multicolumn{2}{c}{GRVS} & \multicolumn{2}{c}{GENE-SPLINE} \\
\midrule
  	& Time & Speedup & Time & Speedup \\
Basic GD & 15.84 (0.41) & 1.0 & 147.78 (1.21) & 1.0 \\
AC 	& 3.84 (0.08) & 4.1 & 8.19 (0.08) & 18.0 \\
SSR & 3.30 (0.11) & 4.8 & 6.34 (0.05) & 23.3 \\
SEDPP  & 3.51 (0.10) & 4.5 & 6.89 (0.05) & 21.4 \\
SSR-BEDPP & \textbf{2.51 (0.10)} & \textbf{6.3} & \textbf{4.42 (0.04)} & \textbf{33.4} \\
\bottomrule
\end{tabular}
\end{adjustbox}
\label{tab_real_res_glasso}
\end{table}

\section{Conclusion} \label{sect_conc}

In this paper, we propose novel, efficient hybrid safe-strong rules (HSSR) for lasso-type models. The key of HSSR is to incorporate a simple, safe rule into SSR to alleviate a large amount of unnecessary post-convergence KKT checking required by SSR. We demonstrate that this hybrid of two very different types of rules is substantially more efficient than either type alone. This innovative idea is motivated by the insights from careful complexity analysis. The idea is simple yet remarkably powerful in reducing the computing time for solving the entire regularization path in lasso-type problems. As a result, the newly proposed rules are much more scalable and suitable to large-scale sparse learning problems.

For the standard lasso problem, we develop two instances of HSSR: SSR-Dome and SSR-BEDPP. Moreover, we extend SSR-BEDPP to the elastic net and the group lasso to illustrate the generalizability of the HSSR framework. Extensive studies on synthetic and real data sets demonstrate that the newly proposed rules substantially outperform the existing state-of-the-art screening rules SSR and SEDPP.

The basic idea proposed in this manuscript is rather general and can be modified or extended in several ways.  Although we concentrated on coordinate descent throughout this manuscript, line 12 of Algorithm~\ref{algo_PCD_HSSR} could use any lasso solver -- nothing specific about coordinate descent is required in order for safe, strong, or hybrid rules to work.  More importantly, the hybrid screening idea can be extended to other sparse modeling problems.  The major limiting factor here is the existence of safe rules.  For many modeling problems, no safe rules have yet been identified.  Some notable exceptions are sparse logistic regression and sparse support vector machines, two classes of problems where we expect hybrid rules would also lead to more efficient algorithms.

One possibility for improving upon the hybrid rules proposed here would be to ``re-hybridize'' SSR with another safe rule once BEDPP is no longer effective. For example, as illustrated in Figure 1, for this data set BEDPP becomes useless at $\lambda_{60}$. At that point, we could apply the EDPP rule~\eqref{SEDPP_rule} to obtain a new safe rule, effective for $\lam_{61}, \lam_{62}, \ldots$.  by only varying $\lambda_{k+1}$. This rule would require $O(np)$ calculations at $\lam_{61}$, but only $O(p)$ calculations at future as the computationally expensive terms only need to be computed once and saved, as in the proposed algorithm.  This approach may offer additional computational savings beyond SSR-BEDPP, especially in the latter part of the solution path.

All screening rules presented in this manuscript are implemented in two publicly accessible R packages \texttt{biglasso} (for the standard lasso and elastic net) and \texttt{grpreg} (for the group lasso). Benchmarking experiments \citep{zeng2017} show that \texttt{biglasso} is considerably faster than existing packages of its kind, including the popular R package \texttt{glmnet}, as a result of the hybrid screening rules proposed here.


\appendix
\section*{Appendix}

\section{Proof of Theorem~\ref{theorem_bedpp}} \label{proof_bedpp_lasso}

\begin{proof}
Since at $\lambda_m$ the dual optimal solution is known: $\bt(\lambda_m) = \frac{\y}{n\lambda_m}$, Theorem 19 in in Wang et al.~\cite{JMLR:v16:wang15a} is applicable. Let $\v_1(\lambda_m) = \text{sign}(\x^T_* \y) \x_*, ~ \v_2(\lambda, \lambda_m) = \frac{\y}{n \lambda} - \frac{\y}{n \lambda_m}, ~ \v_2^\perp(\lambda, \lambda_m) = \v_2(\lambda, \lambda_m) - \frac{\langle \v_1(\lambda_m), \v_2(\lambda, \lambda_m) \rangle}{\| \v_1(\lambda_m) \|^2}\v_1(\lambda_m)$, then the BEDPP rule for the lasso~(\ref{linear_lasso}) (note our lasso formulation has a factor $1/n$) rejects the $j$th feature if 
$
\left|\x_j^T \left(\frac{\y}{n \lambda_m} + \frac{1}{2} \v_2^\perp(\lambda, \lambda_m) \right) \right| < 1 - \frac{1}{2} \|\v_2^\perp(\lambda, \lambda_m) \| \| \x_j \|.
$

Note that: (a) under conditions~(\ref{stand}) $\|\x_j\|=\sqrt{n}, ~ \forall j$; (b) $\x_*^T \y = \text{sign}(\x_*^T \y) n \lambda_m$. With some algebra, it's easy to show that $\v_2^\perp(\lambda, \lambda_m) =
\left(\frac{1}{n \lambda} - \frac{1}{n \lambda_m}\right) (\y - \text{sign}(\x_*^T \y) \lambda_m \x_*)$, and hence $\| \v_2^\perp(\lambda, \lambda_m) \|$ can be simplified as $\left(\frac{1}{n \lambda} - \frac{1}{n \lambda_m}\right) \sqrt{\y^T \y - n \lambda_m^2}$. Substituting these two pieces into the above inequality with some rearrangement yields to the simplified BEDPP. 
\end{proof}

\section{Proof of Theorem~\ref{theorem_sedpp}} \label{proof_sedpp_lasso}
\begin{proof}
In view of Corollary 20 in~\cite{JMLR:v16:wang15a}, for $k=0$ case, $\v_1(\lambda_k)$, $\v_2(\lambda_{k+1}, \lambda_k)$ and $\v_2^\perp(\lambda_{k+1}, \lambda_k)$ reduce to those in Appendix~\ref{proof_bedpp_lasso}. So the SEDPP rule becomes the BEDPP rule.

For $0 < k < K$, let $\v_1(\lambda_k)=\frac{\X \widehat{\bb}(\lambda_k)}{n \lambda_k}$, $\v_2(\lambda_{k+1}, \lambda_k) = \frac{\y}{n \lambda_{k+1}} - \frac{\y - \X \widehat{\bb}(\lambda_k)}{n \lambda_k}$, $\v_2^\perp(\lambda_{k+1}, \lambda_k) = \v_2(\lambda_{k+1}, \lambda_k) - \frac{\langle \v_1(\lambda_k), \v_2(\lambda_{k+1}, \lambda_k) \rangle}{\| \v_1(\lambda_k) \|^2}\v_1(\lambda_k)$. According to Corollary 20 in~\cite{JMLR:v16:wang15a}, the (sequential) EDPP rule for the lasso~(\ref{linear_lasso}) rejects the $j$th feature if 
$
\left|\x_j^T \left(\frac{\y - \X \widehat{\bb}(\lambda_k)}{n \lambda_m} + \frac{1}{2} \v_2^\perp(\lambda_{k+1}, \lambda_k) \right) \right| < 1 - \frac{1}{2} \|\v_2^\perp(\lambda_{k+1}, \lambda_k) \| \| \x_j \|.
$

Denote $c= \frac{\lambda_k - \lambda_{k+1}}{\lambda_k \lambda_{k+1}}$, $a = \y^T \X \widehat{\bb}(\lambda_k)$. We first note that $\v_2^\perp(\lambda_{k+1}, \lambda_k)$ can be simplified as:
\begin{align*}
  \v_2^\perp&(\lambda_{k+1}, \lambda_k) = \v_2(\lambda_{k+1}, \lambda_k) - \frac{\langle \v_1(\lambda_k), \v_2(\lambda_{k+1}, \lambda_k) \rangle}{\| \v_1(\lambda_k) \|^2}\v_1(\lambda_k)\\
  &= \frac{\y}{n \lambda_{k+1}} - \frac{\y - \X \widehat{\bb}(\lambda_k)}{n \lambda_k} - \\
  &\quad \frac{\widehat{\bb}(\lambda_k)^{T} \X^T \left((\lambda_k -  \lambda_{k+1}) \y + \lambda_{k+1}\X \widehat{\bb}(\lambda_k) \right) \X \widehat{\bb}(\lambda_k)}{n \lambda_k \lambda_{k+1} \|\X \widehat{\bb}(\lambda_k) \|^2}\\
  &= \frac{\y}{n \lambda_{k+1}} - \frac{\y - \X \widehat{\bb}(\lambda_k)}{n \lambda_k} - \frac{ac \X \widehat{\bb}(\lambda_k)}{n \|\X \widehat{\bb}(\lambda_k) \|^2} - \frac{ \X \widehat{\bb}(\lambda_k) }{n \lambda_k} \\
  &= \frac{c}{n} \left(\y - \frac{a \X \widehat{\bb}(\lambda_k)}{ \|\X \widehat{\bb}(\lambda_k) \|^2} \right).
\end{align*}
Then with some algebra, $\| \v_2^\perp(\lambda_{k+1}, \lambda_k) \|$ can be simplified to be $\frac{c}{n} \sqrt{ \|\y\|^2 - a^2 / \|\X \widehat{\bb} \|^2}$. Plugging the two terms back into the inequality of the SEDPP rule and with some rearrangement gives the simplified SEDPP rule and completes the proof.
\end{proof}

\section{Proof of Theorem~\ref{theorem_bedpp_elastic_net}}
\label{proof_bedpp_elnet}

\begin{proof}
Denote
$
\tilde{\X} = 
\begin{pmatrix}
\X \\
\sqrt{n(1-\alpha)\lambda} \cdot \I
\end{pmatrix}, ~ \tilde{\y} = 
\begin{pmatrix}
\y \\
\zero
\end{pmatrix}.
$
The the elastic net problem can then be rewritten as,
$$
\widehat{\bb}(\lambda, \alpha) = \argmin_{\bb \in \mathbb{R}^p}\frac{1}{2n} (\tilde{\y} - \tilde{\X} \bb)' (\tilde{\y} - \tilde{\X} \bb) + \alpha \lambda \|\bb\|_1,
$$
which is in the form of the standard lasso with original $\lambda$ reparameterized with $\alpha \lambda$. Hence Theorem 19 in~\cite{JMLR:v16:wang15a} is applicable, provided $(\X, \y, \lambda_m, \lambda)$ is replaced by $(\tilde{\X}, \tilde{\y}, \alpha \lambda_{m}, \alpha \lambda)$. That is, the BEDPP rule rejects the $j$th feature if
\begin{align} \label{rbedpp_enet}
\left|\tilde{\x}_j^T \left(\frac{\tilde{\y}}{n \alpha \lambda_m} + \frac{1}{2} \tilde{\v}_2^\perp(\lambda, \lambda_m) \right) \right| < 1 - \frac{1}{2} \|\tilde{\v}_2^\perp(\lambda, \lambda_m) \| \| \tilde{\x}_j \|.
\end{align}
Here $\lambda_{m}$ is reparameterized as $\lambda_{m} = \max_j | \tilde{\x}_j^T \tilde{\y} / (n \alpha) |$. $\tilde{\v}_2^\perp(\lambda, \lambda_m) = \tilde{\v}_2(\lambda, \lambda_m) - \frac{\langle \tilde{\v}_1(\lambda_m), \tilde{\v}_2(\lambda, \lambda_m) \rangle}{\| \tilde{\v}_1(\lambda_m) \|^2}\tilde{\v}_1(\lambda_m)$, where $\tilde{\v}_1(\lambda_m) = \text{sign}(\tilde{\x}^T_* \tilde{\y})\tilde{\x}_*, ~ \tilde{\v}_2(\lambda, \lambda_m) = \frac{\tilde{\y}}{n \alpha \lambda} - \frac{\tilde{\y}}{n \alpha \lambda_m}$.

On the other hand, it's easy to verify that $\tilde{\x}_j^T \tilde{\y} = \x_j^T \y, ~ \forall j; ~ \| \tilde{\y} \| = \| \y \|$; $ \| \tilde{\x}_j \|^2 = \| \x_j \|^2 + n \lambda(1 - \alpha);~ \tilde{\x}_j^T \tilde{\x}_k = \x_j^T \x_k, \forall j \neq k$. With some algebra, $\tilde{\v}_2^\perp(\lambda, \lambda_m)$ can be simplified as follows,
\begin{align*}
  \tilde{\v}_2^\perp(\lambda, \lambda_m) &= \frac{\tilde{\y}}{n \alpha \lambda} - \frac{\tilde{\y}}{n \alpha \lambda_m} - \\
  &\quad \frac{\text{sign}(\x_*^T \y) \x_*^T \y}{\| \tilde{\x}_* \|^2} \left(\frac{1}{n \alpha \lambda} - \frac{1}{n \alpha \lambda_m}\right) \text{sign}(\x_*^T \y) \tilde{\x}_* \\
  &= \tilde{\y} \left(\frac{1}{n \alpha \lambda} - \frac{1}{n \alpha \lambda_m}\right) -\\
  &\quad \left(\frac{1}{n \alpha \lambda} - \frac{1}{n \alpha \lambda_m}\right)  \frac{\x_*^T \y}{n (1 + \lambda (1 - \alpha))} \tilde{\x}_* \\
&= \left(\frac{1}{n \alpha \lambda} - \frac{1}{n \alpha \lambda_m}\right) \left(\tilde{\y} - \frac{\text{sign}(\x_*^T \y) \alpha \lambda_m}{1 + \lambda (1 - \alpha)} \tilde{\x}_* \right).
\end{align*}
From this, we have the result $\| \tilde{\v}_2^\perp(\lambda, \lambda_m) \| = \left(\frac{1}{n \alpha \lambda} - \frac{1}{n \alpha \lambda_m}\right)  \sqrt{\| \y \|^2 - \frac{n \alpha^2 \lambda_m^2}{1 + \lambda (1 - \alpha)} }$.  Let us now consider two cases:

\begin{itemize}
\item If $\tilde{\x}_j = \tilde{\x}_*$, $\tilde{\x}_j^T\tilde{\x}_* n (1 + \lambda (1 - \alpha))$, it can be shown that
\begin{align*}
  &\left|\tilde{\x}_j^T \left(\frac{\tilde{\y}}{n \alpha \lambda_m} + \frac{1}{2} \tilde{\v}_2^\perp(\lambda, \lambda_m) \right) \right|\\
  &\qquad=\frac{1}{2n \alpha \lambda \lambda_m} \left| (\lambda + \lambda_m) \x_*^T \y - (\lambda_m - \lambda) \text{sign}(\x_*^T \y) n \alpha \lambda_m \right| \\
  &\qquad= 1,
\end{align*}
which is always larger than the RHS of~\eqref{rbedpp_enet}. In other words, $\x_*$ won't be rejected.

\item If $\tilde{\x}_j \neq \tilde{\x}_*$, $\tilde{\x}_j^T\tilde{\x}_* = \x_j^T \x_*$. We have
\begin{multline*}
  \left|\tilde{\x}_j^T \left(\frac{\tilde{\y}}{n \alpha \lambda_m} + \frac{1}{2} \tilde{\v}_2^\perp(\lambda, \lambda_m) \right) \right| = \frac{1}{2n \alpha \lambda \lambda_m} \times \\
  \left| (\lambda + \lambda_m) \x_j^T \y - (\lambda_m - \lambda) \frac{\text{sign}(\x_*^T \y) \alpha \lambda_m}{1 + \lambda (1 - \alpha)} \x_j^T \x_*\right|.
\end{multline*}
\end{itemize}

Plugging this piece and the simplified $\| \tilde{\v}_2^\perp(\lambda, \lambda_m) \|$ into~\eqref{rbedpp_enet} with some additional algebra yields to the BEDPP rule for the elastic net~\eqref{BEDPP_enet}.


\end{proof}

\section{Proof of Theorem~\ref{theorem_BEDPP_glasso}}
\label{proof_bedpp_glasso}

\begin{proof}

We first note that it can be easily shown the dual optimal solution to the group lasso problem~(\ref{linear_glasso}) at $\lambda_m$ is $\bt^*_{\lambda_m} = \frac{\y}{n \lambda_m}$. Denote $\bar{\v} = \X_* \X_*^T \y, ~ \bar{\v}_2(\lambda, \lambda_m) = \frac{\y}{n \lambda} - \bt^*_{\lambda_m}, ~ \bar{\v}_2^\perp(\lambda, \lambda_m)  = \bar{\v}_2(\lambda, \lambda_m) - \frac{\langle \bar{\v}, \bar{\v}_2(\lambda, \lambda_m) \rangle}{\| \bar{\v} \|^2} \bar{\v}$. According to Theorem 20 in Wang et al.~\cite{JMLR:v16:wang15a}, for any $\lambda \in (0, \lambda_m]$, we have the BEDPP rule that rejects the $g$th group of features (i.e., $\widehat{\bb}_g(\lambda) = \zero)$ if,
\begin{align} \label{bedpp_glasso}
\left\| \X_g^T \left( \bt^*_{\lambda_m} + \frac{1}{2} \bar{\v}_2^\perp(\lambda, \lambda_m) \right) \right\| < \sqrt{W_g} - \frac{1}{2} \| \bar{\v}_2^\perp(\lambda, \lambda_m)\| \|\X_g \|.
\end{align}

Note $\bar{\v}_2^\perp(\lambda, \lambda_m)$ can be simplified as follows.
\begin{align*}
& \quad \bar{\v}_2^\perp(\lambda, \lambda_m) = \bar{\v}_2(\lambda, \lambda_m) - \frac{\langle \bar{\v}, \bar{\v}_2(\lambda, \lambda_m) \rangle}{\| \bar{\v} \|^2} \bar{\v} \\
& = \frac{\y}{n} (\frac{1}{\lambda} - \frac{1}{\lambda_m}) - \frac{\y^T \X_* \X_*^T \y \X_* \X_*^T \y}{\y^T \X_* \X_*^T \X_* \X_*^T \y} \frac{1}{n} (\frac{1}{\lambda} - \frac{1}{\lambda_m}) \\
& = \frac{\y}{n} (\frac{1}{\lambda} - \frac{1}{\lambda_m}) - \frac{\y^T \X_* \X_*^T \y \X_* \X_*^T \y}{\y^T \X_* n \X_*^T \y} \frac{1}{n} (\frac{1}{\lambda} - \frac{1}{\lambda_m}) \\
& = \frac{\y}{n} (\frac{1}{\lambda} - \frac{1}{\lambda_m}) - \frac{ \X_* \X_*^T \y}{n} \frac{1}{n} (\frac{1}{\lambda} - \frac{1}{\lambda_m}) \quad (\y^T \X_* \X_*^T \y \text{ is a scalar}) \\
& = \frac{1}{n} (\frac{1}{\lambda} - \frac{1}{\lambda_m}) (\I - \frac{\X_* \X_*^T}{n}) \y,
\end{align*}
where the second equality is because $\X_*^T \X_* = n \I$ under the condition~(\ref{stand_glasso}). The left hand side of the rule then becomes,
\begin{align*}
& \quad \left\| \X_g^T \left( \bt^*(\lambda_m) + \frac{1}{2} \bar{\v}_2^\perp(\lambda, \lambda_m) \right) \right\| \\
& = \left\| \X_g^T \left( \frac{\y}{n \lambda_m} + \frac{1}{2n} (\frac{1}{\lambda} - \frac{1}{\lambda_m}) \left(\I - \frac{\X_* \X_*^T}{n} \right) \y \right) \right\| \\
& = \left\| \frac{\X_g^T \y}{2n} (\frac{1}{\lambda} + \frac{1}{\lambda_m}) - \frac{1}{2n} (\frac{1}{\lambda} - \frac{1}{\lambda_m}) \frac{\X_g^T \bar{\v}}{n} \right\| \\
  & = \frac{1}{2n \lambda \lambda_m} \left\| (\lambda + \lambda_m) \X_g^T \y - (\lambda_m - \lambda) \frac{\X_g^T \bar{\v}}{n} \right\| \\
  & = \frac{1}{2n \lambda \lambda_m} \sqrt{A-B+C}
\end{align*}

The right hand side of the rule is,
\begin{align*}
& \quad \sqrt{W_g} - \frac{1}{2} \| \bar{\v}_2^\perp(\lambda, \lambda_m)\| \|\X_g \| \\
& = \sqrt{W_g} - \frac{1}{2n} (\frac{1}{\lambda} - \frac{1}{\lambda_m}) \| \X_g \| \sqrt{\y^T \left(\I - \frac{\X_* \X_*^T}{n} \right)^T \left(\I - \frac{\X_* \X_*^T}{n} \right) \y} \\
& = \sqrt{W_g} - \frac{1}{2n} (\frac{1}{\lambda} - \frac{1}{\lambda_m}) \| \X_g \| \sqrt{\y^T \left(\I - \frac{1}{n} \X_* \X_*^T \right) \y} \\
& = \sqrt{W_g} - \frac{1}{2n} (\frac{1}{\lambda} - \frac{1}{\lambda_m}) \| \X_g \| \sqrt{\| \y \|^2- n \lambda_m^2 W_* }
\end{align*}
Note that the second equality is due to that $\I - \X_* \X_*^T /n$ is idempotent; the last equality is because: (i) $\| \X_*^T \y \| = n \sqrt{W_*} \lambda_m$, implied by the definitions of $\lambda_m$ and $\X_*$; (ii) $\| \X_g \| = n$ , again implied by the standardization condition~(\ref{stand_glasso}). Here $\|\X_g\|$ is the matrix $L_2$ norm, which is equal to the largest singular value of $\X_g$.

Substituting the simplified results into~(\ref{bedpp_glasso}) with some rearrangement yields the BEDPP rule for the group lasso.

\end{proof}

\section{Effect of correlation on hybrid rules}
\label{App:Cor}

To illustrate the effect of correlation on performance, we carried out a simulation study similar to one in Section \ref{subsect_sim_lasso}, except that: 1) the following are fixed: $n=50,000$, $p=1,000$, and 20 true features; 2) the pairwise correlation between features varies in magnitude from 0 to 0.1, 0.3, 0.5, and 0.8, all with an exchangeable (compound symmetric) structure).  The results are shown in Table \ref{Tab:Cor}.

\begin{table}[h!]
\caption{Average computing time (standard error) for solving the lasso along a sequence of 100 values of $\lambda$ for increasing amounts of pairwise correlation.\label{Tab:Cor}}
\centering
\begin{tabular}{@{}lrrrrr@{}}
\toprule
& 0 & 0.1 & 0.3 & 0.5 & 0.8\tabularnewline
\midrule
Basic PCD & 10.92 (0.23) & 19.63 (0.44) & 28.81 (1.48) & 49.35 (3.51) &
72.21 (5.71)\tabularnewline
AC & 6.15 (0.05) & 6.94 (0.09) & 7.28 (0.20) & 8.35 (0.43) & 8.17
(0.52)\tabularnewline
SSR & 5.37 (0.03) & 6.08 (0.04) & 6.44 (0.18) & 7.59 (0.40) & 9.12
(1.52)\tabularnewline
SEDPP & 5.51 (0.04) & 6.27 (0.06) & 6.62 (0.19) & 7.69 (0.41) & 7.91
(0.49)\tabularnewline
SSR-Dome & 4.03 (0.04) & 4.93 (0.09) & 5.40 (0.20) & 6.90 (0.40) & 8.62
(1.57)\tabularnewline
SSR-BEDPP & 3.22 (0.05) & 4.02 (0.11) & 4.62 (0.21) & 6.11 (0.41) & 8.05
(1.65)\tabularnewline
\bottomrule
\end{tabular}
\end{table}

As the table shows, all algorithms become slower as correlation increases, although some are affected more than others.  The proposed hybrid rules are roughly 30-70\% faster than either strong (SSR) or safe (SEDDP) rules alone when correlation is less than or equal to 0.5.  However, strong rules do become ineffective in the presence of extreme correlation, which thus renders hybrid rules ineffective as well.  Nevertheless, the real data results of Section~\ref{subsect_real_lasso} suggest that hybrid rules are effective for most realistic correlation structures.

\bibliographystyle{ims-nourl}

\end{document}